\documentclass[10pt,letterpaper]{article}
\usepackage[letterpaper,textwidth=5.5in,textheight=9in]{geometry}
\usepackage{times,natbib,microtype}
\usepackage[hyphens]{url}
\usepackage{float}
\usepackage{placeins}
\usepackage[colorlinks=true,linkcolor=blue!40!black,citecolor=blue!40!black,urlcolor=blue!40!black]{hyperref}
\usepackage{graphicx}
\usepackage{amsmath,amssymb,amsthm,mathtools,bm}
\usepackage{booktabs}
\usepackage{array}
\usepackage{multirow}
\usepackage{tikz}
\usetikzlibrary{arrows.meta,positioning,decorations.pathreplacing}
\definecolor{basegray}{HTML}{777777}
\definecolor{depthtwo}{HTML}{82ADD3}
\definecolor{depththree}{HTML}{3379B5}
\definecolor{depthfour}{HTML}{133F70}
\definecolor{regionone}{HTML}{F8E6C9}
\definecolor{regiontwo}{HTML}{DDF0EB}
\definecolor{regionthree}{HTML}{EAE6F5}
\newcommand{\regiontag}[2]{{\setlength{\fboxsep}{2pt}\colorbox{#1}{\strut\hspace{2pt}#2\hspace{2pt}}}}
\newtheorem{proposition}{Proposition}
\newtheorem{lemma}{Lemma}

\newcommand{\E}{\mathbb E}
\newcommand{\R}{\mathbb R}

\newcommand{\sg}{\operatorname{sg}}
\DeclareMathOperator*{\argmin}{arg\,min}
\newcommand{\LocoRegionJointRows}{%
TD3+BC & 66.6 & 41.9 & 19.3 & \textbf{0.0} & 27.8 & 66.7 & 53.7\\
I-MPI-2 & 69.0 & 63.8 & 24.9 & 1.9 & 8.3 & 58.3 & 41.7\\
I-MPI-3 & 68.7 & 71.2 & 33.2 & 1.9 & 5.6 & 50.0 & 35.2\\
I-MPI-4 & \textbf{69.1} & \textbf{79.9} & \textbf{41.4} & 1.9 & \textbf{0.0} & \textbf{37.5} & \textbf{25.0}\\
}

\newcommand{\LocoCollapseRows}{%
0 & 3.7 & 2.8 & 2.8 & 1.9\\
10 & 47.2 & 40.7 & 33.3 & 20.4\\
20 & 53.7 & 41.7 & 35.2 & 25.0\\
30 & 56.5 & 46.3 & 37.0 & 28.7\\
40 & 59.3 & 51.9 & 38.0 & 33.3\\
}

\newcommand{\LocoHighGain}{8.3}
\newcommand{\LocoHighCI}{[2.9,14.3]}
\newcommand{\LocoRescueShare}{65\%}

\newcommand{\IETDHigh}{25.4}
\newcommand{\IEITwoHigh}{41.2}
\newcommand{\IEIThreeHigh}{51.9}
\newcommand{\IEETwoHigh}{38.4}
\newcommand{\IEEThreeHigh}{46.4}

\title{Is One Step Enough for\\Offline Policy Improvement?}
\author{Soohyun Choi$^{*}$ \qquad Seonvin Cho$^{*}$ \qquad Songnam Hong$^{\dagger}$\\[4pt]
Information and Intelligence Systems Laboratory (IISL)\\
Department of Electronic Engineering, Hanyang University\\
Seoul, Republic of Korea\\[3pt]
\texttt{petersun0221@hanyang.ac.kr} \quad \texttt{seonbin0319@hanyang.ac.kr}\\
\texttt{snhong@hanyang.ac.kr}\\[5pt]
{\small $^{*}$Equal contribution. $^{\dagger}$Corresponding author.}}
\date{September 2026}
\hypersetup{pdftitle={Is One Step Enough for Offline Policy Improvement?},
pdfauthor={Soohyun Choi, Seonvin Cho, Songnam Hong}}
\begin{document}
\maketitle

\begin{abstract}
Behavior regularization in offline reinforcement learning limits the
exploitation of critic errors, but strong anchoring can also restrict
policy improvement. We study how policy improvement is composed through
multi-step proximal policy improvement (MPI), which re-centers each
proximal objective on the preceding policy. We parameterize the procedure
by a nominal total horizon $T$ and $K$ stages with local horizon $T/K$,
distinguishing subdivision at a fixed total horizon from additional
refinement at a common local horizon. Our analysis shows that sequential
re-centering can reach endpoints unavailable to any single proximal step
and characterizes how subdivision reduces the leading local discretization
error of ideal updates under a fixed critic. We consider TD3+BC and
IQL-based policy extraction to examine how improvement composition
interacts with actor objectives and policy geometry. TD3+BC experiments
on D4RL locomotion suggest that subdivision can broaden the range of
useful total horizons, while adding refinement stages at a fixed small
local horizon can improve return. The results identify improvement
composition as a design choice alongside regularization strength, with
distinct effects from horizon subdivision and additional policy extraction.
\end{abstract}

\section{Introduction}
\label{sec:introduction}

Offline reinforcement learning improves a policy using a fixed dataset
and learned value estimates. Behavior regularization limits the use of
actions on which these estimates may be unreliable
\citep{fujimoto2019offpolicy,kumar2019bear,wu2019brac}. Its coefficient
sets the trade-off between value maximization and proximity to dataset
actions. Strong anchoring can leave useful improvements unrealized;
weak anchoring can lead the policy into regions where the critic
extrapolates poorly.

Regularization strength alone does not specify how improvement is
composed. In TD3+BC, the actor trades critic value against distance
from dataset actions throughout training. We study
\emph{multi-step proximal policy improvement} (MPI): its first
deterministic actor uses the same data anchor, and each later actor
is anchored to the preceding actor. These are separately trained
actors; later actors use their freshly updated predecessor as the
reference during training.

This construction asks whether a given improvement budget works better
as one actor step or as several smaller steps. We also examine longer
chains built by adding steps of the same size. Under an ideal fixed
critic, a smooth strongly convex example shows that two stages can
reach an endpoint unavailable to any single proximal step. For smooth
ideal updates and a small total budget, subdivision reduces the
leading error relative to the corresponding policy flow.

We test two deterministic actor updates and Gaussian policy extraction
with an IQL critic.
On nine D4RL locomotion tasks, deeper TD3+BC chains retain high
returns over a wider range of tested total horizons. A matched control that
anchors later actors to the first actor finds both gains and losses
from re-centering across tasks. The Gaussian experiments also show
depth-dependent gains and failures across extraction rules.

\section{Background}
\label{sec:background}

\subsection{Offline reinforcement learning}
Consider a discounted Markov decision process
\(\mathcal M=\langle\mathcal S,\mathcal A,P,r,\gamma\rangle\).
Offline reinforcement learning seeks a policy maximizing
\[
J(\pi)=\E_\pi\!\left[\sum_{t\ge0}\gamma^t r(s_t,a_t)\right]
\]
using only a fixed dataset
\(\mathcal D=\{(s_i,a_i,r_i,s_i')\}_{i=1}^N\), without further
environment interaction \citep{levine2020offline}. Because the learned
policy may select actions that are weakly represented in the dataset,
its value estimates can rely on extrapolation; when such estimates enter
Bellman backups, the resulting error can propagate through learning
\citep{fujimoto2019offpolicy,kumar2019bear}.

A common strategy is therefore to regularize policy improvement toward
data-supported actions. Behavior regularization can be imposed directly
on the actor, while conservative value-learning methods instead modify
the critic to discourage unsupported actions
\citep{wu2019brac,kumar2020cql}. In this work, we keep the value-learning
procedure of the base algorithm and study the actor-side improvement
step.

\subsection{Actor-side policy improvement}
Let \(\widehat Q\) denote a critic learned from \(\mathcal D\), and
let \(\rho_{\mathcal D}\) be the dataset state marginal. An actor
update converts the learned value information into a policy while
controlling how far the policy moves from the data-supported region.
The particular extraction rule can materially affect offline RL
performance even when the value-learning procedure is fixed
\citep{park2024bottleneck}.

We consider two extraction patterns used later in the paper.
TD3+BC directly maximizes a learned critic while penalizing squared
deviation from dataset actions \citep{fujimoto2021minimalist}. IQL
learns \(Q\) and \(V\) from dataset transitions and extracts a
stochastic policy by advantage-weighted regression, with
\(A(s,a)=Q(s,a)-V(s)\) \citep{kostrikov2022iql}. We study how
locally regularized actor improvement composes across stages while
retaining each base algorithm's value-learning procedure.

\subsection{Proximal updates and policy geometry}
A proximal update balances optimization of an objective against movement
from a reference point. Given an objective \(F\), a reference \(\nu\),
a metric \(d\), and a step size \(h>0\), a metric proximal step takes
the form
\[
x^+\in\argmin_x
\left\{
F(x)+\frac{1}{2h}d^2(x,\nu)
\right\}.
\]
Smaller \(h\) penalizes movement from \(\nu\) more strongly.
In Euclidean coordinates, a differentiable energy generates a gradient
flow with two familiar time discretizations:
\begin{equation}
\dot x_t=-\nabla F(x_t),\qquad
x_{k+1}^{\rm I}=x_k-h\nabla F(x_{k+1}^{\rm I}),\qquad
x_{k+1}^{\rm E}=x_k-h\nabla F(x_k).
\label{eq:gf_discretizations}
\end{equation}
For Euclidean \(d\), the proximal problem above is the \emph{implicit}
(backward Euler) step: its gradient is evaluated at the new point. The \emph{explicit}
(forward Euler) step evaluates it at the current reference. Repeating
metric proximal steps gives the minimizing-movement construction for
gradient flows \citep{parikh2014proximal,ambrosio2005gradient}; the
Wasserstein JKO scheme is an example \citep{jordan1998variational}.
The implicit and explicit rules use the same vector field but evaluate
it at different actions; this distinction defines the deterministic
actor updates in Section~\ref{sec:det_method}.
The metric fixes policy movement: KL/Fisher geometry yields local
trust-region movement \citep{schulman2015trust}, and Wasserstein
geometry uses transport cost \citep{zhang2018policy,terpin2022trust,pfau2025wpo}.

\section{Single-step proximal improvement}
\label{sec:prox_view}

Hold a learned critic
\(\widehat Q:\mathcal S\times\mathcal A\to\mathbb R\) fixed. We define
the critic-induced policy energy
\begin{equation}
\mathcal E_{\widehat Q}(\pi)
=-\E_{s\sim\rho_{\mathcal D},\,a\sim\pi(\cdot\mid s)}
\widehat Q(s,a).
\label{eq:policy_energy}
\end{equation}
Minimizing this energy corresponds to maximizing the learned critic over
dataset states. For a state distribution \(\rho\) and a statewise policy metric \(d\),
we lift \(d\) to the policy space by
\begin{equation}
\mathsf d_{\rho}^2(\pi,\nu)
:=
\E_{s\sim\rho}
d^2\!\left(\pi(\cdot\mid s),\nu(\cdot\mid s)\right).
\label{eq:policy_metric}
\end{equation}
The statewise metric specifies how two action distributions are compared
at a fixed state, while \(\mathsf d_\rho\) aggregates that movement
over the states relevant to the actor objective.

Given a reference policy \(\nu\), a \emph{single proximal improvement}
(SPI) stage with coefficient \(h>0\) is
\begin{equation}
\mathcal P_h^{\widehat Q}(\nu)
\in\argmin_{\pi\in\Pi}
\left\{
\mathcal E_{\widehat Q}(\pi)
+\frac{1}{2h}\mathsf d_{\rho_{\mathcal D}}^2(\pi,\nu)
\right\}.
\label{eq:prox_operator}
\end{equation}
Thus \(\mathcal P_h^{\widehat Q}\) takes a reference policy and returns a
policy that trades critic improvement against displacement from that
reference. The admissible policy class \(\Pi\) and the geometry determine
the operator. We assume a minimizer exists when discussing exact
proximal maps.

\subsection{TD3+BC as a Wasserstein proximal objective}
For deterministic policies, the Wasserstein--2 distance between Dirac
measures reduces exactly to Euclidean action distance,
\[
W_2^2(\delta_{\mu(s)},\delta_{\nu(s)})
=\|\mu(s)-\nu(s)\|_2^2.
\]
We therefore use the action-mean normalization
\begin{equation}
\mathsf d_{\rho}^2(\mu,\nu)
=\E_{s\sim\rho}\frac{\|\mu(s)-\nu(s)\|_2^2}{n_a},
\label{eq:mean_metric}
\end{equation}
where \(n_a\) is the action dimension. TD3+BC optimizes
\begin{equation}
\mathcal L_{\rm actor}(\mu)
=-\alpha\,\E_s\bar Q(s,\mu(s))
+\E_{(s,a)\sim\mathcal D}
\frac{\|\mu(s)-a\|_2^2}{n_a},
\label{eq:td3bc_background}
\end{equation}
with \(\bar Q=Q_1/C\) and a positive detached normalization \(C\)
\citep{fujimoto2021minimalist}. Averaging the squared cloning term
over dataset actions leaves a distance to their conditional mean,
plus a variance term independent of the actor.

\begin{lemma}[Behavior anchoring as a proximal objective]
\label{prop:prox_equiv}
Assume a fixed critic, a fixed positive scale \(C\), \(\alpha>0\),
finite objective expectations, and finite dataset-action second moments.
On a closed convex action space, the TD3+BC objective has the same
minimizers, when they exist, as Eq.~\eqref{eq:prox_operator} with
the metric in Eq.~\eqref{eq:mean_metric}, deterministic reference
\(\nu(s)=\E_{\mathcal D}[a\mid s]\), energy \(\mathcal E_{\bar Q}\),
and \(h=\alpha/2\).
\end{lemma}

\subsection{Advantage weighting and local Fisher--Rao geometry}
\label{sec:awr_view}
For stochastic policies, the connection is different. Let
\(A(s,a)=Q(s,a)-V(s)\) and consider the KL-regularized distributional
improvement problem from a reference \(\nu(\cdot\mid s)\):
\[
\max_{\pi(\cdot\mid s)}
\left\{
\E_{a\sim\pi(\cdot\mid s)}A(s,a)
-\frac{1}{h}
D_{\mathrm{KL}}\!\left(
\pi(\cdot\mid s)\,\|\,\nu(\cdot\mid s)
\right)
\right\}.
\]
The unrestricted optimizer is proportional to
\(\nu(a\mid s)e^{hA(s,a)}\), motivating advantage-weighted
extraction in AWAC and IQL \citep{nair2021awac,kostrikov2022iql}. Since
\(D_{\mathrm{KL}}\) has second-order term
\(\tfrac12 d_{\rm FR}^2\), the distributional objective is locally
related to a Fisher--Rao proximal step
(Appendix~\ref{app:awr_connection}).

\section{Multi-step proximal policy improvement}
\label{sec:method}

With a fixed critic, initial reference \(\pi_0\), and step sizes
\(h_1,\ldots,h_K\), MPI composes proximal updates:
\begin{equation}
\pi_k\in\mathcal P_{h_k}^{\widehat Q}(\pi_{k-1}),
\qquad k=1,\ldots,K,
\qquad T=\sum_{k=1}^{K}h_k.
\label{eq:mpi_chain}
\end{equation}
For equal steps, \(h=T/K\). Increasing \(K\) at fixed \(T\)
subdivides the nominal horizon; increasing it at fixed \(h\) extends
the total horizon to \(Kh\). These coefficients do not measure actual
action displacement. In practice, we set \(h=T/K\) and train \(K\)
persistent actors in order: the first uses dataset actions or a base
extraction loss, each later actor uses its freshly updated predecessor
as a detached reference, and only \(\pi_K\) is deployed.
Figure~\ref{fig:path} contrasts the ideal single-stage and
re-centered paths.

\begin{figure}[t]
\centering
\begin{tikzpicture}[
  x=1cm,y=1cm,scale=1.1,transform shape,
  box/.style={draw,rounded corners,minimum width=.85cm,minimum height=.5cm,font=\small},
  arr/.style={-{Latex[length=1.8mm]},thick},
  lab/.style={font=\small}]
\node[box] (b1) at (0,1) {\(\pi_0\)};
\node[box] (one) at (6.1,1) {\(\pi^{\rm SPI}\)};
\draw[arr] (b1) -- node[above,lab]{one proximal problem, coefficient \(T\)} (one);
\node[box] (b2) at (0,0) {\(\pi_0\)};
\node[box] (a1) at (1.9,0) {\(\pi_1\)};
\node[box] (a2) at (3.9,0) {\(\cdots\)};
\node[box] (ak) at (6.1,0) {\(\pi_K\)};
\draw[arr] (b2) -- node[above,lab]{\(h\)} (a1);
\draw[arr] (a1) -- node[above,lab]{\(h\)} (a2);
\draw[arr] (a2) -- node[above,lab]{\(h\)} (ak);
\end{tikzpicture}
\caption{Single-step proximal improvement (SPI) at coefficient \(T\)
(top) and \(K\) re-centered MPI steps at \(h=T/K\) (bottom). Each
lower arrow uses the preceding policy as its reference.}
\label{fig:path}
\end{figure}
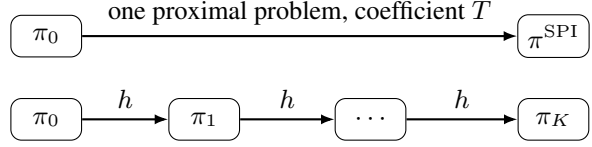

\subsection{Deterministic policy refinement}
\label{sec:det_method}
The two actor updates approximate the implicit and explicit steps in
Eq.~\eqref{eq:gf_discretizations}: the implicit loss evaluates the
critic at the new actor action, while the explicit update regresses
onto a target formed from the critic gradient at the preceding action.

For \emph{implicit-proximal MPI} (I-MPI), \(\nu_1\) is the
sampled dataset action; for \(k\ge2\), \(\nu_k\) is the detached,
freshly updated predecessor. With
\(\bar Q_k^{\rm I}=\widehat Q/C_k^{\rm I}\) and a positive detached
scale, the minibatch objective is
\begin{equation}
\mathcal L_k^{\rm I}(\mu_k\mid\nu_k)
=-2h\,\E_i\bar Q_k^{\rm I}(s_i,\mu_k(s_i))
+\E_i\frac{\|\mu_k(s_i)-\nu_k(s_i)\|_2^2}{n_a}.
\label{eq:actorloss}
\end{equation}
Here \(\E_i\) denotes the minibatch average. Dividing by \(2h\)
gives the proximal form for later stages. At \(K=1\),
I-MPI is TD3+BC with \(\alpha=2T\).

\emph{Explicit-target MPI} (E-MPI) regresses onto the
projected critic-gradient target
\begin{equation}
a_k^{\rm tar}(s)=\Pi_{\mathcal A}
\left[\nu_k(s)+n_a h\frac{\nabla_a\widehat Q(s,\nu_k(s))}{C_k^{\rm E}}\right].
\label{eq:linearized_target}
\end{equation}
The scale \(C_k^{\rm E}>0\) and the target are detached;
\(\Pi_{\mathcal A}\) projects onto the action space.
Appendix~\ref{app:algorithms} specifies their scales,
update order, and regression loss.

In both implementations, the first actor supplies the target
actions for critic bootstrapping; later actors only affect policy
extraction.
We write I-MPI-\(K\) and E-MPI-\(K\) for chains with \(K\) actors
using the respective updates. For example, I-MPI-4 trains four actors
and evaluates the final actor \(\mu_4\).

\subsection{Gaussian policy refinement}
\label{sec:iql_method}
IQL's value targets use dataset actions and do not depend on the
actor \citep{kostrikov2022iql}. We can therefore vary extraction
while retaining its value-learning procedure. We consider a
diagonal-Gaussian actor with a critic-gradient and behavior-cloning
loss (Gaussian \(Q\)+BC), and advantage-weighted regression (AWR),
as first-stage objectives.

Write \(A=Q-V\) and \(\ell_\pi(s,a)=-\log\pi(a\mid s)\). The
first actors minimize
\begin{equation}
\begin{aligned}
\mathcal L^{\rm QBC}_1&=-\E_{s\sim\rho_{\mathcal D},\,a\sim\pi_1}Q(s,a)
                       +h^{-1}\E_{\mathcal D}\ell_{\pi_1}(s,a),\\
\mathcal L^{\rm AWR}_1&=\E_{\mathcal D}
                [e^{hA(s,a)}\ell_{\pi_1}(s,a)].
\end{aligned}
\label{eq:iql_base}
\end{equation}
Gaussian \(Q\)+BC extends DDPG+BC-style extraction
\citep{park2024bottleneck} by evaluating expected \(Q\) under a
learned-variance Gaussian and using dataset negative log-likelihood.
The first-stage losses are extraction objectives, whereas each later
actor uses a Gaussian proximal loss with its detached, freshly updated
predecessor as reference: \(W_2\) after \(Q\)+BC and Fisher--Rao after
AWR. The first-stage AWR connection to Fisher--Rao is local
(Section~\ref{sec:awr_view}).
Appendix~\ref{app:gaussian_geometry} gives the distances and local
updates; Appendix~\ref{app:iql_implementation} gives the training
details.

\section{Properties of composed updates}
\label{sec:analysis}

\subsection{Endpoints beyond one proximal step}
Changing the coefficient of one proximal problem retains its
reference. A second proximal step uses the first endpoint as its
reference and can change the set of reachable endpoints.

For an energy \(E\) on \(\R^2\), write
\(\mathcal P_h^E(x)=\argmin_y\{E(y)+\|y-x\|_2^2/(2h)\}\) and define
the endpoints of exactly \(K\) steps by
\[
\mathcal R_K(E,x_0)
=\{(\mathcal P_{h_K}^E\circ\cdots\circ\mathcal P_{h_1}^E)(x_0)
:h_1,\ldots,h_K>0\}.
\]
\begin{proposition}[Strict enlargement of reachable endpoints]
\label{prop:noncollapse}
There exist a smooth strongly convex \(E:\R^2\to\R\) and
\(x_0\in\R^2\) such that
\[
\overline{\mathcal R_1(E,x_0)}
\subsetneq\overline{\mathcal R_2(E,x_0)}.
\]
\end{proposition}

Appendix~\ref{app:reachable} proves the inclusion and gives a
quadratic example for which it is strict. The reachable sets range
over positive step sizes; this proposition establishes a difference
between update compositions, not a fixed-\(T\) performance ordering.

\subsection{Local horizon subdivision}
With a fixed critic,
unrestricted statewise optimization, and inactive action constraints,
the deterministic flow and its implicit and explicit steps satisfy
\begin{equation}
\begin{aligned}
\partial_t\mu_t(s)
&=n_a\nabla_a\bar Q(s,\mu_t(s))
=:f_{\bar Q}(\mu_t)(s),\\
\mu^{\rm I}
&=\nu+T f_{\bar Q}(\mu^{\rm I}),\\
\mu^{\rm E}
&=\Pi_{\mathcal A}
\left[\nu+T f_{\bar Q}(\nu)\right].
\end{aligned}
\label{eq:flow_discretizations}
\end{equation}
These are the backward and forward steps in
Eq.~\eqref{eq:gf_discretizations}, specialized to the deterministic
actor. Appendix~\ref{app:behavior_anchor} derives the relations and
distinguishes the fixed reference from a sampled first-stage target.

At a fixed state, write \(x\in\mathbb R^{n_a}\) for the action and
\(f(x)=n_a\nabla_a\bar Q(s,x)\), with Jacobian \(J_f(x)\).
Let \(\Phi_T(x)\) be its time-\(T\) flow. The critic and its
normalization stay fixed across steps.

\begin{samepage}
\begin{proposition}[Local horizon composition]
\label{prop:fixed_budget_composition}
Suppose an ideal explicit or implicit update is a local map
\[
\Psi_h(x)=x+h f(x)+h^2p(x)+O(h^3),
\]
with \(f\in C^2\), \(p\in C^1\), and a remainder uniform on a
neighborhood containing the local flow and update paths.
Assume action constraints are inactive there.
For fixed \(K\ge1\) and \(T\to0\), define
\[
\Delta_2(x):=\frac12 J_f(x)f(x)-p(x).
\]
Then
\begin{equation}
\begin{aligned}
\Psi_{T/K}^{K}(x)-\Phi_T(x)
&=-\frac{T^2}{K}\Delta_2(x)+O(T^3),\\
\Psi_{T/K}^{K}(x)-\Psi_T(x)
&=\frac{K-1}{K}T^2\Delta_2(x)+O(T^3).
\end{aligned}
\label{eq:direct_composed_gap}
\end{equation}
\end{proposition}
\end{samepage}

Explicit Euler has \(p=0\), while backward Euler has
\(p(x)=J_f(x)f(x)\). Both share the first-order displacement
\(Tf(x)\); at fixed \(K\) and small \(T\), subdivision reduces
the leading error relative to the common flow by \(1/K\).
The second relation measures the difference from one direct step.
The expansion assumes exact statewise maps and a fixed critic and
normalization throughout the chain. It does not establish convergence
as \(K\to\infty\) at fixed \(T\); Appendix~\ref{app:composition}
gives the proof and further scope conditions.

\subsection{Critic and optimization errors}
For approximate actor solves and an imperfect critic, the proximal
objectives give the following fixed-critic bound.

\begin{proposition}[Inexact improvement under a fixed critic]
\label{prop:inexact_improvement}
Let \(\widehat F=-\mathcal E_{\widehat Q}\) be fixed throughout a
chain with \(h_k>0\). Assume the preceding policy is feasible at
each stage and the attained proximal objective is at most its
value at that preceding policy plus \(\xi_k\ge0\).
For a reference objective \(F\), suppose
\(|\widehat F(\pi)-F(\pi)|\le\epsilon\) on the visited policies.
Then
\begin{equation}
F(\pi_K)-F(\pi_0)
\ge\sum_{k=1}^{K}\frac{\mathsf d_{\rho_{\mathcal D}}^2(\pi_k,\pi_{k-1})}{2h_k}
-\sum_{k=1}^{K}\xi_k-2\epsilon.
\label{eq:margin}
\end{equation}
\end{proposition}

Telescoping the stagewise inequalities gives the displacement term
minus accumulated optimization error; comparing the two endpoints
adds \(2\epsilon\) (Appendix~\ref{app:critic_error}). The result
bounds the change in \(F\) under the stated critic and optimization
errors, without assuming those errors are small in a trained actor.

\section{Experiments}
\label{sec:experiments}
\begingroup
\setlength{\textfloatsep}{10pt}
\setlength{\intextsep}{8pt}

We study improvement composition on nine D4RL locomotion tasks
\citep{fu2020d4rl}, with four training seeds per observed configuration.
Task-level scores average seeds; aggregate curves weight tasks equally.
Deterministic scores use the final checkpoint after \(10^6\) updates.
A fixed-\(T\) TD3+BC sweep tests horizon subdivision: all depths share
the critic-update count, while depth \(K\) uses \(K\) actor optimizer
calls per delayed update. We then test both gradient-flow
discretizations and Gaussian extraction with an IQL critic.

\subsection{Horizon and depth in TD3+BC}
\label{sec:local_budget}

TD3+BC reaches its highest sampled mean at \(T=1.5\). After
I-MPI-4's peak at \(T=4\), its mean stays at or above 90\% of that
peak through \(T=10\) and first drops below at \(T=12\). We define
\(R_1\) by \(0<T\le1.5\), \(R_2\) by \(1.5<T\le10\), and \(R_3\) by
\(10<T\le40\) on the sampled grid. These boundaries describe the
aggregate curve rather than a common threshold for every task.

\begin{figure}[H]
\centering
\includegraphics[width=.80\textwidth]{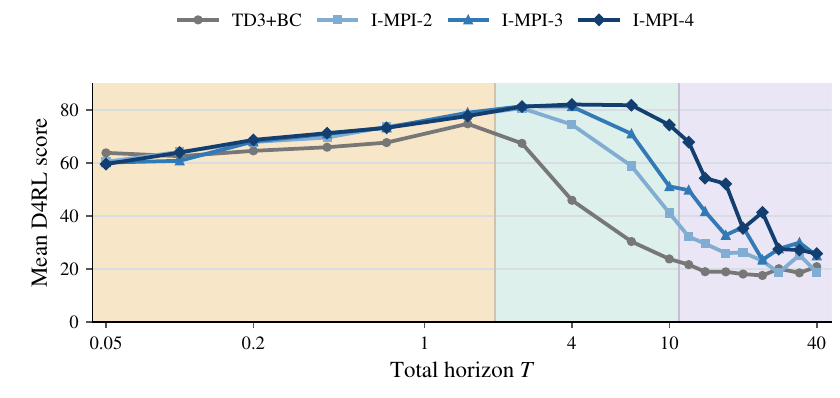}
\caption{Mean D4RL score by total horizon \(T\). Shading marks
\(R_1\), \(R_2\), and \(R_3\) as defined in the text.}
\label{fig:stability}
\end{figure}

In \(R_1\), returns are similar across depths. In \(R_2\) and \(R_3\), I-MPI-4
has a higher mean than TD3+BC. Across \(R_2\) and \(R_3\) combined (108
task--horizon cells), the fraction with four-seed mean below 20 is
25.0\% for I-MPI-4 versus 53.7\% for TD3+BC; I-MPI-4 exceeds
I-MPI-3 by \(\LocoHighGain\) points (95\% task-bootstrap interval
\(\LocoHighCI\)). Its task-level mean exceeds TD3+BC on eight of
nine tasks, with HalfCheetah medium as the exception
(Appendix~\ref{app:task_results}).
\begin{table}[H]
\centering
\caption{Regional mean D4RL score and percentage of low-return cells.
Higher mean scores and lower low-return percentages are better.}
\label{tab:regions}
\small\setlength{\tabcolsep}{6pt}
\begin{tabular}{@{}lccc@{\hspace{1.8em}}cccc@{}}
\toprule
& \multicolumn{3}{c}{Mean D4RL score} & \multicolumn{4}{c}{Low-return cells (\%)}\\
\cmidrule(lr){2-4}\cmidrule(l){5-8}
Method & \regiontag{regionone}{\(R_1\)} & \regiontag{regiontwo}{\(R_2\)} & \regiontag{regionthree}{\(R_3\)}
& \regiontag{regionone}{\(R_1\)} & \regiontag{regiontwo}{\(R_2\)} & \regiontag{regionthree}{\(R_3\)} & \(R_2+R_3\)\\
\midrule
\LocoRegionJointRows
\bottomrule
\end{tabular}
\par\smallskip
{\footnotesize \(R_1\), \(R_2\), and \(R_3\) contain 54, 36, and 72 task--horizon cells,
respectively; horizons have equal weight. Bold marks the best displayed
value in each column, including rounded ties.}
\end{table}

At a common small local horizon \(h=0.1\), the two-stage chain
\((K,T)=(2,0.2)\) averages 67.9, versus 62.6 for the one-stage
\((K,T)=(1,0.1)\) actor; here adding a stage also extends the total
horizon.

At fixed \(T\), increasing depth reduces the first actor's
coefficient \(h=T/K\) and changes the critic targets it supplies.
Within I-MPI-4 runs, the endpoint scores \(9.13\) points above its
first actor on average in \(R_2\); over the sampled \(12\le T\le20\)
part of \(R_3\), their mean difference is \(-0.23\). In a matched \(K=4\)
control, re-centering has positive and negative task-level effects
relative to a fixed reference (mean \(-0.73\) points;
Appendices~\ref{app:extraction} and~\ref{app:recenter_k4}).

\subsection{Implicit and explicit actor updates}
\label{sec:implicit_explicit}

On the shared \(4\le T\le20\) grid, increasing depth from
\(K=2\) to \(K=3\) raises the nine-task mean from
\(\IEETwoHigh\) to \(\IEEThreeHigh\) for E-MPI and from
\(\IEITwoHigh\) to \(\IEIThreeHigh\) for I-MPI.
TD3+BC averages \(\IETDHigh\) on the same grid
(Figure~\ref{fig:implicit_explicit}).

\begin{figure}[!htbp]
\centering
\includegraphics[width=.86\textwidth]{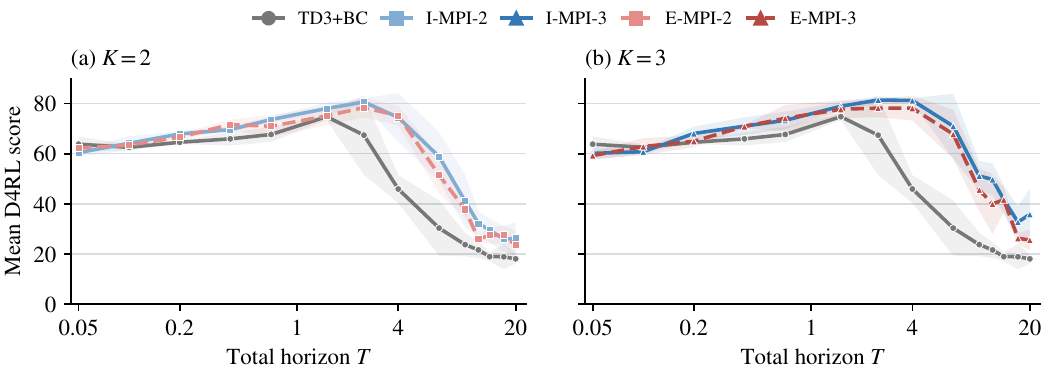}
\caption{I-MPI and E-MPI at matched depth and total horizon.
Bands show one sample standard deviation across seed-specific task means;
TD3+BC has \(K=1\).}
\label{fig:implicit_explicit}
\end{figure}

\subsection{Gaussian extraction with IQL critics}
\label{sec:iql_results}

IQL's actor-independent value targets let us vary extraction
without changing its value-learning procedure. Gaussian AWR--FR
responds differently across tasks (Figure~\ref{fig:iql_fr_main}).
Hopper medium-replay improves with depth around \(T=2.5\). On Hopper
and HalfCheetah expert at \(T=1.5\), the single-stage policies remain
strong where deeper endpoints collapse. Thus FR refinement helps in
some task--horizon settings but cannot be read as a uniform depth gain.

\begin{figure}[H]
\centering
\includegraphics[width=.92\textwidth,trim=0 4pt 0 7pt,clip]{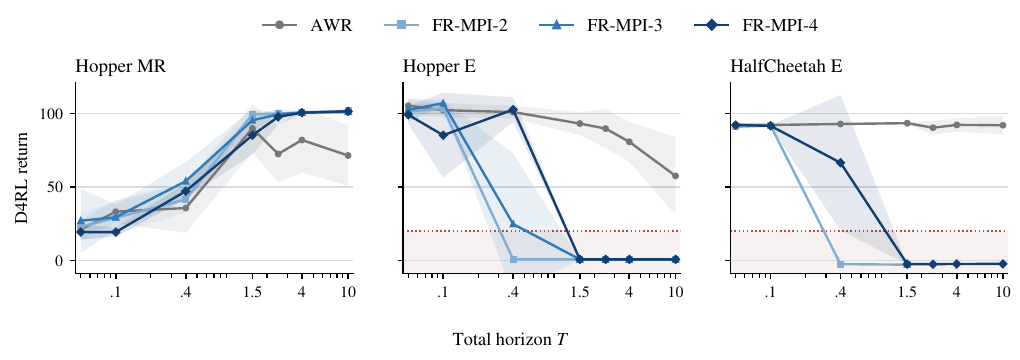}
\caption{Gaussian AWR and Fisher--Rao refinements with an IQL critic.
MR and E denote medium-replay and expert; bands show one sample standard
deviation across seeds. Missing configurations are blank; the shaded
region in expert panels marks returns below 20.}
\label{fig:iql_fr_main}
\end{figure}

With Gaussian \(Q\)+BC--\(W_2\), greater depth raises the nine-task
mean over intermediate horizons; all depths lose the gain at the
largest \(T\) (Figure~\ref{fig:iql_qbc_main}).

\begin{figure}[H]
\centering
\includegraphics[width=.78\textwidth,trim=0 8pt 0 4pt,clip]{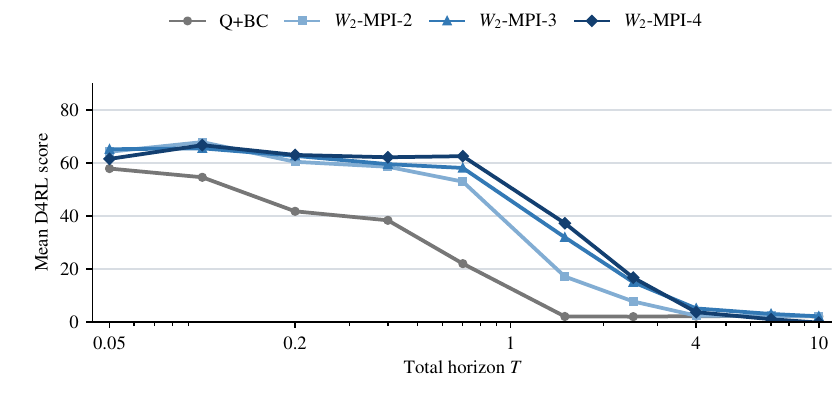}
\caption{Gaussian \(Q\)+BC and \(W_2\) refinements with an IQL critic.
The curves use the same measured horizon grid.}
\label{fig:iql_qbc_main}
\end{figure}

On Hopper and HalfCheetah expert at \(T=1.5\), Gaussian
\(Q\)+BC--\(W_2\) also fails, including at \(K=1\), whereas the
AWR single-stage actor retains high return. These configurations
change both the extraction loss and the refinement geometry; equal
nominal \(T\) does not calibrate their movement scales
(Appendix~\ref{app:iql_implementation}).

\subsection{TD3+BC target and endpoint diagnostics}
\label{sec:td3bc_diagnostics}

A separate matched diagnostic at five horizons compares the I-MPI-4
first actor supplying bootstrap targets, the TD3+BC target actor, and
the deployed I-MPI-4 endpoint. The first actor lies closer to observed successor actions
than the TD3+BC target actor, while the endpoint moves farther from
those actions (Figure~\ref{fig:matched_exposure}). Thus a smaller
target-action displacement can coexist with a more distant deployed
policy. The diagnostic measures movement relative to recorded actions;
it does not establish that either actor has a more accurate critic.

\begin{figure}[H]
\centering
\includegraphics[width=.80\textwidth,trim=0 10pt 0 7pt,clip]{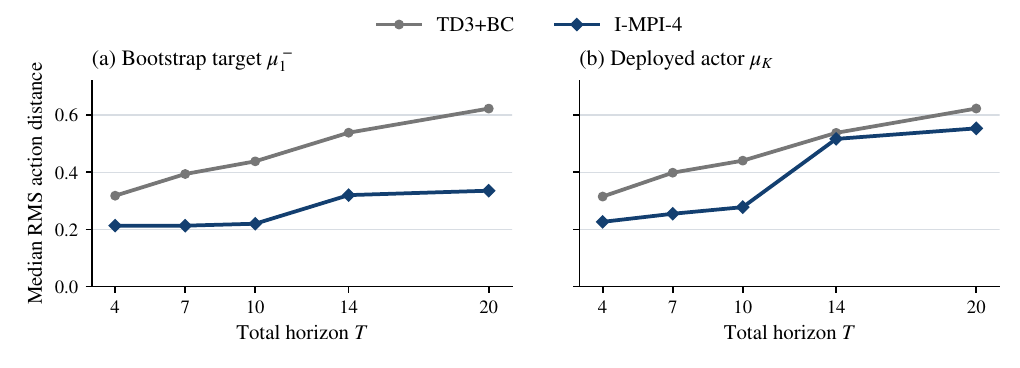}
\caption{RMS distance to observed successor actions: (a) bootstrap
target, (b) deployed actor. Gray: TD3+BC; blue: I-MPI-4. Points are
medians across the matched tasks and two diagnostic seeds.}
\label{fig:matched_exposure}
\end{figure}

\endgroup

\section{Related work}
\label{sec:related}

\paragraph{Behavior regularization and policy extraction.}
BCQ, BEAR, and BRAC limit departure from offline data
\citep{fujimoto2019offpolicy,kumar2019bear,wu2019brac};
TD3+BC and ReBRAC use quadratic actor penalties
\citep{fujimoto2021minimalist,tarasov2023rebrac}, whereas AWAC and IQL
use advantage-weighted extraction \citep{nair2021awac,kostrikov2022iql}.
Extraction matters even with fixed value learning \citep{park2024bottleneck}.

\paragraph{One-step and iterative offline improvement.}
One-step RL extracts from a behavior-policy value estimate
\citep{brandfonbrener2021offline}; iterative critic regularization can
coincide with it under specific assumptions
\citep{eysenbach2023connection}. STR uses behavior-supported trust regions
\citep{mao2023supported}, CFPI gives constrained one-step and iterative
operators \citep{li2023closedform}, and CPI re-centers KL regularization
across policy-iteration rounds \citep{ma2024iteratively}. MPI composes
actors under a common critic within each update, with total horizon
$T$ and depth $K$ varied separately; its first deterministic actor
supplies critic targets.

\paragraph{Policy geometry and expressive actors.}
Information-geometric trust regions and Wasserstein policy flows
provide movement costs
\citep{amari2000methods,schulman2015trust,zhang2018policy,moskovitz2020efficient,terpin2022trust,pfau2025wpo}.
Value Gradient Flow implements Wasserstein minimizing movement through
value-gradient particle transport with separate training and evaluation
budgets \citep{xu2026vgf}. MPI learns persistent actor stages, tests
subdivision at fixed \(T\), and deploys the final actor without
test-time transport. Diffusion-QL and FQL enrich actor parameterization
\citep{wang2023diffusion,park2025fql}; MPI changes the update sequence
within an actor family.

\section{Discussion and conclusion}
\label{sec:limitations}
\label{sec:conclusion}

Splitting a nominal horizon changes the reachable endpoints of ideal
proximal updates and the behavior of trained actors. On the measured
TD3+BC sweep, I-MPI-4 retains high return over more horizons and has
fewer low-return cells in \(R_2\) and \(R_3\). At fixed \(T\), its first actor
also uses a smaller coefficient and supplies the bootstrap targets.
That branch moves less relative to dataset actions while the endpoint
need not stay equally close (Appendix~\ref{app:target_movement}). The
fixed-reference control has mixed task-level effects, so the aggregate
gain cannot be attributed solely to re-centering. At a common small
local horizon, adding a second stage can raise the nine-task mean as
the total horizon grows.

Depth helps on some Gaussian extraction tasks and fails on others.
The high-return regions were identified retrospectively, and selecting
a horizon on two dynamics families gives mixed held-out outcomes
(Appendix~\ref{app:task_results}). The flow analysis assumes ideal
updates with a fixed critic and scale; training changes both the critic
and its normalization. Composition therefore matters in the tested
setting, but depth alone does not select a reliable horizon. Training
uses extra actors; inference deploys only the endpoint.

\label{sec:main_end}

\subsection*{AI use statement}
Generative AI tools assisted with drafting and polishing manuscript
text, literature retrieval, formulating and checking mathematical
proofs, experiment-design feedback, analysis code, interpretation of
logged results, and figure preparation. The research idea originated
with the authors, who are responsible for the final scientific claims
and supporting evidence.

\bibliography{references}

@book{amari2000methods,
  author = {Amari, Shun-ichi and Nagaoka, Hiroshi},
  title = {Methods of Information Geometry},
  publisher = {American Mathematical Society},
  year = {2000}
}

@article{gelbrich1990wasserstein,
  author = {Gelbrich, Matthias},
  title = {On a Formula for the {L2} {Wasserstein} Metric between Measures on {Euclidean} and {Hilbert} Spaces},
  journal = {Mathematische Nachrichten},
  volume = {147},
  pages = {185--203},
  year = {1990},
  doi = {10.1002/mana.19901470121}
}

@book{hairer1993solving,
  author = {Hairer, Ernst and N{\o}rsett, Syvert P. and Wanner, Gerhard},
  title = {Solving Ordinary Differential Equations {I}: Nonstiff Problems},
  edition = {2},
  publisher = {Springer},
  year = {1993},
  doi = {10.1007/978-3-540-78862-1}
}

@article{kurtek2015bayesian,
  author = {Kurtek, Sebastian and Bharath, Karthik},
  title = {Bayesian Sensitivity Analysis with the {Fisher--Rao} Metric},
  journal = {Biometrika},
  volume = {102},
  number = {3},
  pages = {601--616},
  year = {2015},
  doi = {10.1093/biomet/asv026}
}

@book{ambrosio2005gradient,
  author = {Ambrosio, Luigi and Gigli, Nicola and Savar\'e, Giuseppe},
  title = {Gradient Flows in Metric Spaces and in the Space of Probability Measures},
  publisher = {Birkh\"auser},
  year = {2005}
}

@inproceedings{brandfonbrener2021offline,
  author = {Brandfonbrener, David and Whitney, William and Ranganath, Rajesh and Bruna, Joan},
  title = {Offline {RL} Without Off-Policy Evaluation},
  booktitle = {Advances in Neural Information Processing Systems},
  year = {2021}
}

@misc{fu2020d4rl,
  author = {Fu, Justin and Kumar, Aviral and Nachum, Ofir and Tucker, George and Levine, Sergey},
  title = {{D4RL}: Datasets for Deep Data-Driven Reinforcement Learning},
  year = {2020},
  eprint = {2004.07219},
  archivePrefix = {arXiv}
}

@inproceedings{fujimoto2021minimalist,
  author = {Fujimoto, Scott and Gu, Shixiang Shane},
  title = {A Minimalist Approach to Offline Reinforcement Learning},
  booktitle = {Advances in Neural Information Processing Systems},
  year = {2021}
}

@inproceedings{fujimoto2019offpolicy,
  author = {Fujimoto, Scott and Meger, David and Precup, Doina},
  title = {Off-Policy Deep Reinforcement Learning without Exploration},
  booktitle = {Proceedings of the 36th International Conference on Machine Learning},
  year = {2019}
}

@article{jordan1998variational,
  author = {Jordan, Richard and Kinderlehrer, David and Otto, Felix},
  title = {The Variational Formulation of the {Fokker--Planck} Equation},
  journal = {SIAM Journal on Mathematical Analysis},
  volume = {29},
  number = {1},
  pages = {1--17},
  year = {1998},
  doi = {10.1137/S0036141096303359}
}

@inproceedings{kumar2019bear,
  author = {Kumar, Aviral and Fu, Justin and Tucker, George and Levine, Sergey},
  title = {Stabilizing Off-Policy {Q}-Learning via Bootstrapping Error Reduction},
  booktitle = {Advances in Neural Information Processing Systems},
  year = {2019}
}

@article{moskovitz2020efficient,
  author = {Moskovitz, Ted and Arbel, Michael and Husz\'ar, Ferenc and Gretton, Arthur},
  title = {Efficient Wasserstein Natural Gradients for Reinforcement Learning},
  journal = {arXiv preprint arXiv:2010.05380},
  year = {2020}
}

@inproceedings{pfau2025wpo,
  author = {Pfau, David and Davies, Ian and Borsa, Diana L. and Ara{\'u}jo, Jo{\~a}o G. M. and Tracey, Brendan D. and van Hasselt, Hado},
  title = {Wasserstein Policy Optimization},
  booktitle = {Proceedings of the 42nd International Conference on Machine Learning},
  series = {Proceedings of Machine Learning Research},
  volume = {267},
  pages = {49128--49149},
  year = {2025}
}

@inproceedings{schulman2015trust,
  author = {Schulman, John and Levine, Sergey and Moritz, Philipp and Jordan, Michael I. and Abbeel, Pieter},
  title = {Trust Region Policy Optimization},
  booktitle = {Proceedings of the 32nd International Conference on Machine Learning},
  year = {2015}
}

@misc{tarasov2022corl,
  author = {Tarasov, Denis and Nikulin, Alexander and Akimov, Dmitry and Kurenkov, Vladislav and Kolesnikov, Sergey},
  title = {{CORL}: Research-Oriented Deep Offline Reinforcement Learning Library},
  howpublished = {NeurIPS Offline Reinforcement Learning Workshop},
  year = {2022}
}

@article{terpin2022trust,
  author = {Terpin, Antonio and Lanzetti, Nicolas and Yardim, Batuhan and D\"orfler, Florian and Ramponi, Giorgia},
  title = {Trust Region Policy Optimization with Optimal Transport Discrepancies: Duality and Algorithm for Continuous Actions},
  journal = {arXiv preprint arXiv:2210.11137},
  year = {2022}
}

@inproceedings{zhang2018policy,
  author = {Zhang, Ruiyi and Chen, Changyou and Li, Chunyuan and Carin, Lawrence},
  title = {Policy Optimization as Wasserstein Gradient Flows},
  booktitle = {Proceedings of the 35th International Conference on Machine Learning},
  year = {2018}
}

@article{wu2019brac,
  author = {Wu, Yifan and Tucker, George and Nachum, Ofir},
  title = {Behavior Regularized Offline Reinforcement Learning},
  journal = {arXiv preprint arXiv:1911.11361},
  year = {2019}
}

@inproceedings{tarasov2023rebrac,
  author = {Tarasov, Denis and Kurenkov, Vladislav and Nikulin, Alexander and Kolesnikov, Sergey},
  title = {Revisiting the Minimalist Approach to Offline Reinforcement Learning},
  booktitle = {Advances in Neural Information Processing Systems},
  year = {2023}
}

@inproceedings{kumar2020cql,
  author = {Kumar, Aviral and Zhou, Aurick and Tucker, George and Levine, Sergey},
  title = {Conservative {Q}-Learning for Offline Reinforcement Learning},
  booktitle = {Advances in Neural Information Processing Systems},
  year = {2020}
}

@inproceedings{nair2021awac,
  author = {Nair, Ashvin and Gupta, Abhishek and Dalal, Murtaza and Levine, Sergey},
  title = {{AWAC}: Accelerating Online Reinforcement Learning with Offline Datasets},
  booktitle = {International Conference on Learning Representations},
  year = {2021}
}

@inproceedings{kostrikov2022iql,
  author = {Kostrikov, Ilya and Nair, Ashvin and Levine, Sergey},
  title = {Offline Reinforcement Learning with Implicit {Q}-Learning},
  booktitle = {International Conference on Learning Representations},
  year = {2022}
}

@inproceedings{wang2023diffusion,
  author = {Wang, Zhendong and Hunt, Jonathan J. and Zhou, Mingyuan},
  title = {Diffusion Policies as an Expressive Policy Class for Offline Reinforcement Learning},
  booktitle = {International Conference on Learning Representations},
  year = {2023}
}

@inproceedings{park2025fql,
  author = {Park, Seohong and Li, Qiyang and Levine, Sergey},
  title = {Flow {Q}-Learning},
  booktitle = {Proceedings of the 42nd International Conference on Machine Learning},
  year = {2025}
}

@misc{xu2026vgf,
  author = {Xu, Haoran and Hu, Kaiwen and Sojoudi, Somayeh and Zhang, Amy},
  title = {Reinforcement Learning via Value Gradient Flow},
  year = {2026},
  eprint = {2604.14265},
  archivePrefix = {arXiv},
  primaryClass = {cs.LG}
}

@inproceedings{park2024bottleneck,
  title={Is Value Learning Really the Main Bottleneck in Offline {RL}?},
  author={Park, Seohong and Frans, Kevin and Levine, Sergey and Kumar, Aviral},
  booktitle={Advances in Neural Information Processing Systems},
  year={2024}
}

@article{levine2020offline,
  author = {Levine, Sergey and Kumar, Aviral and Tucker, George and Fu, Justin},
  title = {Offline Reinforcement Learning: Tutorial, Review, and Perspectives on Open Problems},
  journal = {arXiv preprint arXiv:2005.01643},
  year = {2020}
}

@article{parikh2014proximal,
  author = {Parikh, Neal and Boyd, Stephen},
  title = {Proximal Algorithms},
  journal = {Foundations and Trends in Optimization},
  volume = {1},
  number = {3},
  pages = {127--239},
  year = {2014}
}

@inproceedings{eysenbach2023connection,
  author = {Eysenbach, Benjamin and Geist, Matthieu and Levine, Sergey and Salakhutdinov, Ruslan},
  title = {A Connection between One-Step {RL} and Critic Regularization in Reinforcement Learning},
  booktitle = {Proceedings of the 40th International Conference on Machine Learning},
  series = {Proceedings of Machine Learning Research},
  volume = {202},
  pages = {9485--9507},
  year = {2023},
  url = {https://proceedings.mlr.press/v202/eysenbach23a.html}
}

@inproceedings{li2023closedform,
  author = {Li, Jiachen and Zhang, Edwin and Yin, Ming and Bai, Qinxun and Wang, Yu-Xiang and Wang, William Yang},
  title = {Offline Reinforcement Learning with Closed-Form Policy Improvement Operators},
  booktitle = {Proceedings of the 40th International Conference on Machine Learning},
  series = {Proceedings of Machine Learning Research},
  volume = {202},
  pages = {20485--20528},
  year = {2023},
  url = {https://proceedings.mlr.press/v202/li23av.html}
}

@inproceedings{ma2024iteratively,
  author = {Ma, Yi and Hao, Jianye and Hu, Xiaohan and Zheng, Yan and Xiao, Chenjun},
  title = {Iteratively Refined Behavior Regularization for Offline Reinforcement Learning},
  booktitle = {Advances in Neural Information Processing Systems},
  volume = {37},
  year = {2024},
  url = {https://proceedings.neurips.cc/paper_files/paper/2024/file/663bce02a0050c4a11f1eb8a7f1429d3-Paper-Conference.pdf}
}

@inproceedings{mao2023supported,
  author = {Mao, Yixiu and Zhang, Hongchang and Chen, Chen and Xu, Yi and Ji, Xiangyang},
  title = {Supported Trust Region Optimization for Offline Reinforcement Learning},
  booktitle = {Proceedings of the 40th International Conference on Machine Learning},
  series = {Proceedings of Machine Learning Research},
  volume = {202},
  pages = {23829--23851},
  year = {2023},
  url = {https://proceedings.mlr.press/v202/mao23c.html}
}
\clearpage
\appendix
\setcounter{equation}{0}
\setcounter{proposition}{0}
\setcounter{lemma}{0}
\setcounter{table}{0}
\setcounter{figure}{0}
\renewcommand{\theequation}{S\arabic{equation}}
\renewcommand{\theproposition}{S\arabic{proposition}}
\renewcommand{\thelemma}{S\arabic{lemma}}
\renewcommand{\thetable}{S\arabic{table}}
\renewcommand{\thefigure}{S\arabic{figure}}
\renewcommand{\theHequation}{supp.\arabic{equation}}
\renewcommand{\theHproposition}{supp.\arabic{proposition}}
\renewcommand{\theHlemma}{supp.\arabic{lemma}}
\renewcommand{\theHtable}{supp.\arabic{table}}
\renewcommand{\theHfigure}{supp.\arabic{figure}}

\section*{Appendix}
\label{sec:appendix_start}
\raggedbottom

\section{Proofs and geometric derivations}
\label{app:proofs}

\subsection{Behavior anchoring: proof of Lemma~\ref{prop:prox_equiv}}
\label{app:behavior_anchor}
\begin{equation}
\E_{(s,a)}\frac{\|\mu(s)-a\|_2^2}{n_a}
=\mathsf d_{\rho_{\mathcal D}}^2(\mu,b)+\mathrm{const}.
\label{eq:bc_decomposition}
\end{equation}

Let $b(s)=\mathbb E[a\mid s]$. The conditional bias--variance identity
\[
\mathbb E[\|\mu(s)-a\|_2^2\mid s]
=\|\mu(s)-b(s)\|_2^2+
\mathbb E[\|a-b(s)\|_2^2\mid s]
\]
proves Eq.~\eqref{eq:bc_decomposition} after averaging over states and
dividing by $n_a$. The second term is independent of $\mu$.
Dividing the anchored objective
$-\alpha\mathbb E_s\bar Q(s,\mu(s))+\mathsf d_{\rho_{\mathcal D}}^2(\mu,\nu)$
by $\alpha$ gives Eq.~\eqref{eq:prox_operator} with coefficient $h=\alpha/2$ and normalized energy.
Under the metric inner product
$n_a^{-1}\mathbb E_s\langle\cdot,\cdot\rangle$, the negative energy gradient is
$f_{\bar Q}(\mu)(s)=n_a\nabla_a\bar Q(s,\mu(s))$.
The usual proximal optimality condition \citep{parikh2014proximal} gives the
corresponding update: for unrestricted statewise optimization, an interior differentiable minimizer
satisfies $(\mu^+-\nu)/T=f_{\bar Q}(\mu^+)$; evaluating the same vector
field at $\nu$ gives the explicit step in Eq.~\eqref{eq:flow_discretizations}.

The conditional-mean anchor is admissible on a closed convex action space.
The result assumes a nonempty minimizer set; continuity on a compact
admissible class suffices. The same rescaling applies to a
ReBRAC-style squared actor penalty with fixed critic and coefficient
\citep{tarasov2023rebrac}.

\paragraph{Sampled first hop of E-MPI.}
\label{app:sampled_explicit}
For a fixed positive scale $C$, unrestricted population MSE regression to
the sampled explicit targets yields
\[
\mu_1^{\rm E}(s)=\mathbb E\!\left[
\Pi_{\mathcal A}\!\left(a+\frac{n_a h}{C}\nabla_a\widehat Q(s,a)\right)
\,\middle|\,s\right].
\]
Without clipping, this is
$b(s)+(n_a h/C)\mathbb E[\nabla_a\widehat Q(s,a)\mid s]$,
which generally differs from a gradient step at $b(s)$.
Thus the sampled first E-MPI hop differs from an explicit step at $b$.

\subsection{Proof of Proposition~\ref{prop:noncollapse}}
\label{app:reachable}
Using the standard proximal-map definition \citep{parikh2014proximal},
for any smooth strongly convex \(E\), the proximal maps are unique and
\(\mathcal P_\epsilon^E(y)\to y\) as \(\epsilon\downarrow0\). Indeed,
optimality and the lower bound on \(E\) give
\[
\|\mathcal P_\epsilon^E(y)-y\|_2^2
\le 2\epsilon\bigl(E(y)-\inf E\bigr).
\]
Every \(z=\mathcal P_h^E(x_0)\in\mathcal R_1(E,x_0)\) is therefore the
limit of \(\mathcal P_\epsilon^E(z)\in\mathcal R_2(E,x_0)\), proving
\(\overline{\mathcal R_1(E,x_0)}\subseteq
\overline{\mathcal R_2(E,x_0)}\).

Consider \(E(x)=\tfrac12(x_1^2+2x_2^2)\) and \(x_0=(1,1)\).
Direct minimization gives
\begin{equation}
\mathcal P_h^E(x_0)=\left(\frac1{1+h},\frac1{1+2h}\right),
\qquad (\mathcal P_1^E\circ\mathcal P_1^E)(x_0)
=\left(\frac14,\frac19\right).
\label{eq:unreachable_endpoint}
\end{equation}
If a sequence of single-step endpoints converged to
\((1/4,1/9)\), its first coordinate would force the corresponding
coefficients to converge to \(3\); its second coordinate would then
converge to \(1/7\), a contradiction. Hence
\((1/4,1/9)\in\mathcal R_2(E,x_0)\setminus
\overline{\mathcal R_1(E,x_0)}\), and the inclusion is strict.

\subsection{Proof of Proposition~\ref{prop:fixed_budget_composition}}
\label{app:composition}
We use the local truncation expansion of forward and backward Euler
steps \citep{hairer1993solving}. Under the proposition's smoothness, uniform-remainder, and interiority
assumptions, write $f_0=f(x_0)$, $J_0=J_f(x_0)$, and $p_0=p(x_0)$.
For fixed $K$, consistency gives $x_j=x_0+jhf_0+O(h^2)$ and hence
$f(x_j)=f_0+jhJ_0f_0+O(h^2)$. Summing the increments gives
\begin{align*}
x_K-x_0
&=h\sum_{j=0}^{K-1}f(x_j)+h^2\sum_{j=0}^{K-1}p(x_j)+O(Kh^3)\\
&=Khf_0+h^2\frac{K(K-1)}{2}J_0f_0+Kh^2p_0+O(h^3).
\end{align*}
Substituting $h=T/K$ yields
\[
\Psi_{T/K}^{K}(x_0)
=x_0+Tf_0+T^2\!\left[\frac{K-1}{2K}J_0f_0+\frac{1}{K}p_0\right]+O(T^3).
\]
The exact flow has expansion
$\Phi_T(x_0)=x_0+Tf_0+\tfrac12T^2J_0f_0+O(T^3)$.
Subtracting it, or the single-step expansion
$\Psi_T(x_0)=x_0+Tf_0+T^2p_0+O(T^3)$, proves the two identities in
Eq.~\eqref{eq:direct_composed_gap}.
Explicit Euler has $p_0=0$; Taylor expansion of the local backward-Euler
branch gives $p_0=J_0f_0$.

At fixed $T$, convergence as $K\to\infty$ additionally requires a unique
flow on $[0,T]$, sufficient smoothness around the full flow and numerical
paths, and stable steps. For ReLU critics, the action gradient is constant
inside an activation region. There the explicit update and same-region
stationary implicit branch coincide, although a global proximal minimizer
can lie in another region. The smooth expansion does not apply across an
activation boundary.

\subsection{Proof of Proposition~\ref{prop:inexact_improvement}}
\label{app:critic_error}
Let $\pi_{k-1}$ belong to the feasible class of stage $k$. It suffices
that the approximate solve satisfies
\[
-\widehat F(\pi_k)+\frac{\mathsf d_{\rho_{\mathcal D}}^2(\pi_k,\pi_{k-1})}{2h_k}
\le -\widehat F(\pi_{k-1})+\xi_k.
\]
An objective value within $\xi_k$ of the proximal infimum implies
this inequality, the attained-objective condition used in the proof.
Summing it telescopes the fixed learned objective. Applying the
error bound at $\pi_0$ and $\pi_K$ gives Eq.~\eqref{eq:margin}.
Only endpoint critic errors enter the bound, while optimization errors
accumulate across stages. A uniform action-value error bound implies
the assumed objective error under the fixed dataset state marginal.
The bound also applies to feasible stochastic policies when the
objective and normalization remain fixed throughout the chain.
The live TD3+BC sweep changes normalization scales across stages.

\subsection{Gaussian geometry for MPI refinement}
\label{app:gaussian_geometry}
For diagonal Gaussians \(\pi=\mathcal N(m,\operatorname{diag}(\sigma^2))\)
and \(\nu=\mathcal N(\widetilde m,\operatorname{diag}(\widetilde\sigma^2))\),
the Wasserstein distance follows from the Gaussian formula
\citep{gelbrich1990wasserstein}:
\begin{equation}
W_2^2(\pi,\nu)
=\|m-\widetilde m\|_2^2+\|\sigma-\widetilde\sigma\|_2^2.
\label{eq:gaussian_w2}
\end{equation}
Its Dirac limit gives deterministic refinement. The square-root-density
representation gives the ambient Fisher--Rao distance
\citep{amari2000methods,kurtek2015bayesian}; with the normalization
whose infinitesimal squared line element is the full Fisher information
metric, for positive Gaussian densities it takes the form
\begin{equation}
d_{\rm FR}(\pi,\nu)
=2\arccos\!\left(\int\sqrt{\pi(a)\nu(a)}\,da\right).
\label{eq:fr_distance}
\end{equation}
This is the ambient density-space distance evaluated on Gaussians,
rather than the intrinsic geodesic distance within the Gaussian family.

For diagonal Gaussians with positive standard deviations, the
Bhattacharyya coefficient in Eq.~\eqref{eq:fr_distance} is
\begin{equation}
B(\pi,\nu)=\prod_{j=1}^{n_a}
\left(\frac{2\sigma_j\widetilde\sigma_j}
{\sigma_j^2+\widetilde\sigma_j^2}\right)^{1/2}
\exp\!\left[-\frac{(m_j-\widetilde m_j)^2}
{4(\sigma_j^2+\widetilde\sigma_j^2)}\right].
\label{eq:gaussian_bc}
\end{equation}
Thus $d_{\rm FR}^2=4\arccos^2 B$ is evaluated analytically.
The shortest ambient path need not be Gaussian. Both distances use
sums over action coordinates; rescaling a distance changes its matching
coefficient $h$.

The diagonal-Gaussian proximal minimizer may differ from the
unrestricted distributional solution; optimizing within the family
is generally not equivalent to projecting the unrestricted solution.

To describe local motion at a fixed state, set
$G(m,\sigma)=\E_\epsilon\widehat Q(s,m+\sigma\odot\epsilon)$,
where $\epsilon\sim\mathcal N(0,I)$. Under differentiation beneath
the expectation,
\[
g_m=\nabla_mG=\E_\epsilon\nabla_a\widehat Q(s,m+\sigma\odot\epsilon),
\qquad
g_\sigma=\nabla_\sigma G
=\E_\epsilon[\epsilon\odot\nabla_a\widehat Q(s,m+\sigma\odot\epsilon)].
\]
The local squared line elements are
\[
 ds_{W_2}^2=\sum_j(dm_j^2+d\sigma_j^2),\qquad
 ds_{\rm FR}^2=\sum_j\frac{dm_j^2+2d\sigma_j^2}{\sigma_j^2}.
\]
\begin{lemma}[Local Gaussian refinement]
\label{lem:gaussian_motion}
Assume \(G\) is smooth with differentiation beneath the expectation,
all reference variances are positive, and action and parameter constraints
are inactive. For a smooth local proximal branch within the
diagonal-Gaussian family converging to the reference as \(h\to0\),
using the coordinate-sum distances above,
the first-order updates are
\begin{equation}
\begin{array}{lll}
W_2: & \Delta m=h g_m+O(h^2), & \Delta\sigma=h g_\sigma+O(h^2),\\[2pt]
\mathrm{FR}: & \Delta m=h\sigma^2\odot g_m+O(h^2), &
\Delta\sigma=\tfrac h2\sigma^2\odot g_\sigma+O(h^2).
\end{array}
\label{eq:gaussian_local_motion}
\end{equation}
All quantities on the right are evaluated at the reference policy,
with the same critic and coefficient.
\end{lemma}

\begin{proof}
For \(\theta=(m,\sigma)\), write the local squared-distance expansion as
\(d^2(\theta+\delta,\theta)=\delta^\top M(\theta)\delta+O(\|\delta\|^3)\).
The stationary equation for \(-G+d^2/(2h)\) on the local branch gives
\(\delta=hM^{-1}\nabla G+O(h^2)\).
For \(W_2\), \(M=I\); for FR,
\(M=\operatorname{diag}(\sigma^{-2},2\sigma^{-2})\).
Substitution proves Eq.~\eqref{eq:gaussian_local_motion}.
\end{proof}

A small variance attenuates the FR mean velocity locally.
Using $\widehat Q(s,m)$ gives $g_\sigma=0$. For Gaussian $W_2$
refinement with a mean-only energy, the optimal standard deviation
is the reference standard deviation.

\subsection{Local Fisher--Rao connection to advantage weighting}
\label{app:awr_connection}
The local Fisher metric and KL expansion are standard
\citep{amari2000methods}. For a regular parametric density family and a small parameter
increment $\delta$, let $I(\theta)$ denote its Fisher information
matrix. Normalization cancels the linear term in a Taylor expansion
of either KL orientation, giving
\[
\begin{aligned}
D_{\rm KL}(p_{\theta+\delta}\|p_\theta)
&=\tfrac12\delta^\top I(\theta)\delta+o(\|\delta\|^2),\\
D_{\rm KL}(p_\theta\|p_{\theta+\delta})
&=\tfrac12\delta^\top I(\theta)\delta+o(\|\delta\|^2),\\
d_{\rm FR}^2(p_{\theta+\delta},p_\theta)
&=\delta^\top I(\theta)\delta+o(\|\delta\|^2).
\end{aligned}
\]
These expansions require smooth positive densities and a regular
parameterization near the reference. Replacing $d_{\rm FR}^2/(2h)$ locally by
$D_{\rm KL}(\pi\|\nu)/h$ yields the distributional problem
\[
\max_\pi\;\E_\pi A(s,a)-h^{-1}D_{\rm KL}(\pi\|\nu),
\qquad
\pi_h(a\mid s)=\frac{\nu(a\mid s)e^{hA(s,a)}}{Z_h(s)}.
\]
The formula assumes a finite normalizer and support within $\nu$.
Gaussian weighted maximum likelihood fits the advantage-weighted
distribution, while finite sampling and clipped weights further
distinguish the implemented AWR stage from an exact distributional
update \citep{nair2021awac,kostrikov2022iql}. Later FR stages use the exact ambient distance in
Eqs.~\eqref{eq:fr_distance} and \eqref{eq:gaussian_bc}.

\section{Implementation and evaluation protocols}
\label{app:experiments}

\subsection{Deterministic algorithms and schedules}
\label{app:algorithms}
The horizon--depth implementation builds on TD3+BC and CORL
\citep{fujimoto2021minimalist,tarasov2022corl}.
It uses $T=\tau=\alpha/2$ and $h=T/K$.
The $K$ actors are independently initialized persistent networks with separate
Adam states. Each iteration updates the critic. Every second iteration,
actors take one optimizer step each on the same minibatch in chain order.
After actor one updates, its target actor and the target critics receive
soft updates with coefficient $0.005$; later actors use the updated
references in Eq.~\eqref{eq:actor_references}.
At actor update \(t\), the exact minibatch references are
\begin{equation}
\nu_{1,t}(s_i)=a_i,
\qquad
\nu_{k,t}(s_i)=\sg\mu_{k-1,t}^{+}(s_i),\quad k\ge2,
\label{eq:actor_references}
\end{equation}
where \(+\) denotes the freshly updated output and \(\sg\) stops gradients.
These time indices are suppressed in the main-text objectives.
E-MPI uses the target in Eq.~\eqref{eq:linearized_target} with
\(\widehat Q_t,\nu_{k,t},C_{k,t}^{\rm E}\) and minimizes
\begin{equation}
\mathcal L_{k,t}^{\rm E}(\mu_k\mid\nu_{k,t})
=\E_i\frac{\|\mu_k(s_i)-a_{k,t}^{\rm tar}(s_i)\|_2^2}{n_a}.
\label{eq:explicit_regression}
\end{equation}
Actor losses use $\widehat Q=Q_1$; Bellman targets use the minimum of the twin
target critics. E-MPI treats its projected regression targets as fixed.

Writing $\mu_{1,t}^{\rm pre}$ for actor one's output before its update,
the detached normalization scales are
\begin{equation}
\begin{aligned}
C_{1,t}^{\rm I}&=\sg\E_i|\widehat Q_t(s_i,\mu_{1,t}^{\rm pre}(s_i))|,\\
C_{1,t}^{\rm E}&=\sg\E_i|\widehat Q_t(s_i,a_i)|,\\
C_{k,t}^{r}&=\sg\E_i|\widehat Q_t(s_i,\nu_{k,t}(s_i))|,
\quad k\ge2,\quad r\in\{\mathrm I,\mathrm E\}.
\end{aligned}
\label{eq:qnorm}
\end{equation}
Each scale includes an additive $10^{-6}$ safeguard. E-MPI's first targets
use sampled dataset actions, as described in Appendix~\ref{app:sampled_explicit}.

We use a 256--256 ReLU actor MLP with tanh output and twin 256--256 ReLU
critic MLPs, batch size 256, Adam learning rate $3\times10^{-4}$,
discount $0.99$, target noise $0.2$, and noise clipping at $0.5$ times
the action bound. States are normalized by the dataset mean and standard
deviation plus $10^{-3}$. Timeout rows are omitted from transition
construction; terminal flags stop bootstrapping.

The nine datasets combine \texttt{halfcheetah}, \texttt{hopper}, and
\texttt{walker2d} with \texttt{medium-v2}, \texttt{medium-replay-v2},
and \texttt{expert-v2}. We use
\[
T\in\{0.05,0.1,0.2,0.4,0.7,1.5,2.5,4,7,10,12,14,17,20,24,28,34,40\}.
\]
Each deterministic sweep score averages ten evaluation episodes per seed.
E-MPI is reported through $T=20$.
In tables, I$K$ and E$K$ denote I-MPI-$K$ and E-MPI-$K$.

\subsection{Gaussian IQL implementation}
\label{app:iql_implementation}
In Eq.~\eqref{eq:iql_base}, \(Q\) is the minimum target critic,
without the detached normalization used by TD3+BC, and
\(\pi_k=\mathcal N(m_k,\operatorname{diag}(\sigma_k^2))\).
The likelihood-BC weight is \(1/h\), and the AWR exponent coefficient
is \(h\). In the reported AWR runs, the exponential weight in
Eq.~\eqref{eq:iql_base} is capped at \(100\):
\(w_h(s,a)=\min\{e^{hA(s,a)},100\}\).
Negative log-likelihood (NLL) sums action dimensions.
The two first-stage losses differ from their later proximal objectives;
Appendix~\ref{app:awr_connection} derives the local AWR--FR relation.

Gaussian \(Q\)+BC uses an unsquashed mean and learned state-dependent
standard deviation, initialized to one. AWR uses a tanh mean and
learned state-dependent standard deviation. Both subsequent refinement
losses evaluate expected \(Q\) with reparameterized Gaussian samples;
the implementation uses eight antithetic samples by default
and stops gradients through the updated predecessor. The current
implementation also updates an independent full-\(T\) control when
\(K>1\), so actor-update cost need not equal that of the TD3+BC chain.
Geometry is computed
between the underlying Gaussians. Mean-action simulator evaluation clips
actions to the environment bounds.

Equal \(h\) does not calibrate action displacement or critic scales:
Gaussian distances and negative log-likelihood sum over action
coordinates, whereas the deterministic squared penalty averages them.

\subsection{Action and target-value diagnostics}
\paragraph{Distances to dataset actions and Bellman targets.}
On \(m\) fixed nonterminal validation transitions, we measure RMS distance
from the observed successor action \(a_{i+1}\):
\begin{equation}
D(\pi)=\left(\frac{1}{m n_a}\sum_{i=1}^{m}
\bigl\|\pi(s_i')-a_{i+1}\bigr\|_2^2\right)^{1/2}.
\label{eq:matched_rms_distance}
\end{equation}
The first target actor \(\mu_1^-\) and endpoint \(\mu_K\) are
evaluated on the same next states.
The squared distance includes conditional action variance.

To measure the corresponding value-space perturbation, we
use a matched TD3+BC target critic
\(F(s',a)=\min_j Q_j^-(s',a)\), held common across methods:
\begin{equation}
\Delta_y^{\rm RMS}
=\gamma\Bigl(\E\bigl[
\bigl(F(s',\widetilde\mu_1^-(s'))
-F(s',\widetilde a_{\mathcal D}')\bigr)^2
\bigr]\Bigr)^{1/2}.
\label{eq:target_value_diagnostic}
\end{equation}
The paired actions receive identical smoothing noise and
box clipping. Their RMS value difference measures bootstrap-target
perturbation under the shared learned critic.

\paragraph{Matched comparison.}
The matched comparison uses nine medium, medium-replay, and expert tasks,
$T\in\{4,7,10,14,20\}$, and two seeds.
Target and endpoint actions are evaluated on identical next states and paired
successor dataset actions. Rows with a timeout or true terminal flag are
excluded. The target-noise diagnostic uses 4096 next-state transitions and 64
draws of clipped smoothing noise per checkpoint. The value-space comparison
in Eq.~\eqref{eq:target_value_diagnostic} uses the same-$T$ TD3+BC target
critic and identical smoothing noise for both compared actions.

\subsection{Aggregation and uncertainty}
\label{app:uncertainty}
The region means in Table~\ref{tab:regions}
average four seeds, then tasks and sampled horizons with equal weight.
The means refer to the measured grid. A low-return cell has a four-seed
task--horizon mean below 20; Table~\ref{tab:collapse_sensitivity}
varies this threshold over the large-horizon range.

For 95\% percentile intervals, we average the specified seeds and horizons
within each task and resample the nine task-level contrasts 100,000 times.
Low-return frequencies are also aggregated within task before resampling.
These intervals retain the within-task seed and horizon structure; the
nine datasets share only three dynamics families.
The first/endpoint comparison resamples nine complete tasks 100,000 times,
retaining within-run actor pairs, horizons, and
training seeds. Evaluation episodes are averaged within each run.

\subsection{Gaussian evaluation and coverage}
\label{app:iql_protocol}
\paragraph{Coverage.}
The Gaussian curves report final-actor mean-action scores.
Gaussian \(Q\)+BC--Wasserstein covers the same nine datasets.
AWR--FR covers seven datasets on an
incomplete horizon--depth grid; Figure~\ref{fig:iql_fr_main}
shows the observed cells for Hopper medium-replay,
Hopper expert, and HalfCheetah expert.

\paragraph{Comparison and aggregation.}
The \(K=1\) curves use separately trained full-\(T\) actors;
the first actor of a deeper chain has coefficient \(T/K\).
Bands show one sample standard deviation. The main-text curves include
only complete configurations; missing configurations
are left blank. The nine-task \(Q\)+BC--Wasserstein mean
weights tasks equally at each shared horizon.

\section{Deterministic task results and controls}
\label{app:results}

\subsection{Task-level horizon responses}
\label{app:task_results}
Figure~\ref{fig:env_curves_prox} resolves the mean trends of Figure~\ref{fig:stability}
by task. Across R2 and R3, I-MPI-4 exceeds TD3+BC on eight of nine
task means. Its mean gain over I-MPI-3 is $\LocoHighGain$ points
($\LocoHighCI$). Cells whose four-seed mean rises from below 20
with I-MPI-3 to at least 20 with I-MPI-4 contribute
\LocoRescueShare{} of the aggregate signed score gain.
Table~\ref{tab:collapse_sensitivity}
shows the response to other low-return thresholds.
Figure~\ref{fig:env_curves_lin} gives E-MPI's task-level responses
through $T=20$.

\begin{table}[t]
\caption{Collapse-rate sensitivity across R2 and R3 ($1.5<T\le40$).
Entries are percentages of task--horizon cells whose four-seed mean
falls below each score threshold.}
\centering
\small
\begin{tabular}{r|rrrr}
\toprule
Threshold & TD3+BC & I2 & I3 & I4\\
\midrule
\LocoCollapseRows
\bottomrule
\end{tabular}
\label{tab:collapse_sensitivity}
\end{table}

\begin{figure*}[t]
\centering
\includegraphics[width=\textwidth]{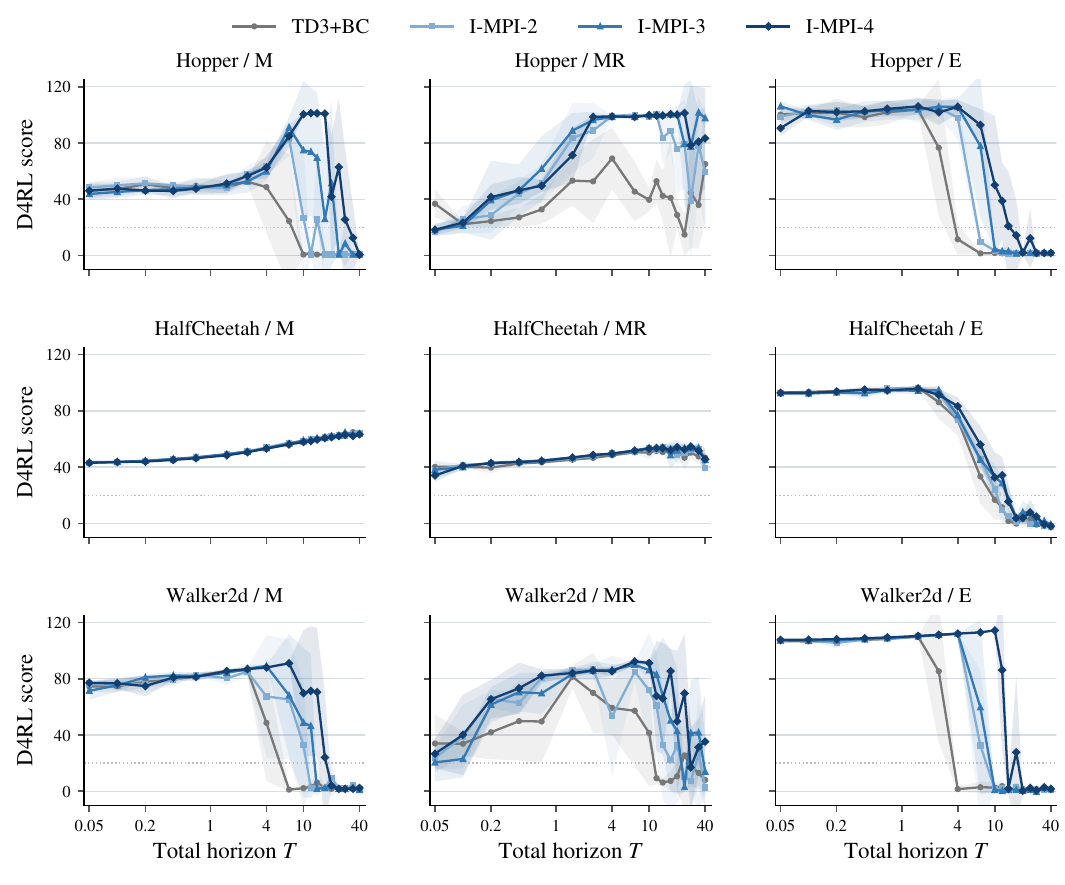}
\caption{I-MPI task-level sweeps. Bands show one sample standard
deviation across seeds; the dotted line marks score $20$.}
\label{fig:env_curves_prox}
\end{figure*}

\begin{figure*}[t]
\centering
\includegraphics[width=\textwidth]{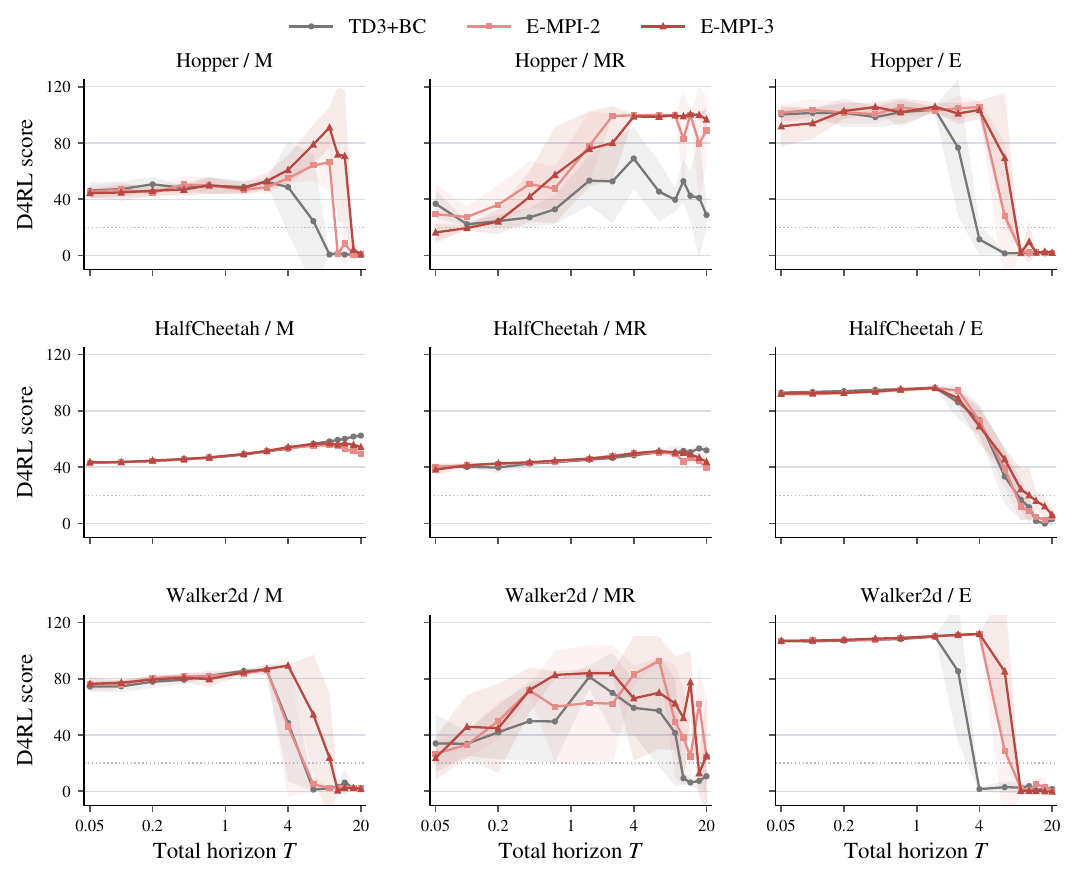}
\caption{E-MPI task-level sweeps through $T=20$. Bands show one sample
standard deviation across seeds; TD3+BC uses the same grid.}
\label{fig:env_curves_lin}
\end{figure*}

\paragraph{Held-out family horizon selection.}
For a retrospective horizon-selection check, we choose each method's
highest-mean horizon on two dynamics families and evaluate it on the
third. The held-out family scores average $78.4$ for I-MPI-4 and
$74.8$ for TD3+BC. I-MPI-4 minus TD3+BC is $+17.1$ on Hopper,
$-8.9$ on HalfCheetah, and $+2.8$ on Walker2d.

\subsection{First actor versus endpoint}
\label{app:extraction}
Figure~\ref{fig:first_endpoint} compares I-MPI-4's first actor and
endpoint at the same checkpoints. Across all fourteen observed horizons
through $T=20$, the endpoint gains $5.76$ points on average.
Within R2 ($T\in\{2.5,4,7,10\}$), its mean rises from $70.43$
to $79.56$, a gain of $9.13$ points with a 95\% task-bootstrap
interval of $[-5.4,23.6]$.
Within the observed part of R3 ($T\in\{12,14,17,20\}$), the
means are $52.42$ and $52.19$, a difference of $-0.23$ points
with interval $[-16.8,15.0]$.

\begin{figure}[H]
\centering
\includegraphics[width=.96\textwidth]{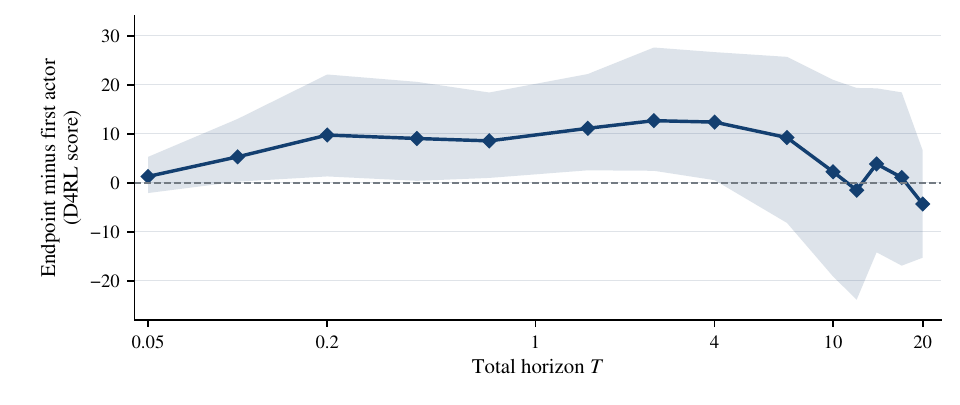}
\caption{I-MPI-4 endpoint minus first-actor return at sampled horizons.
The band is a 95\% task-bootstrap interval retaining paired runs
within each task.}
\label{fig:first_endpoint}
\end{figure}

\subsection{Fixed-reference control for re-centering}
\label{app:recenter_k4}
The first-to-endpoint comparison combines additional actor stages
with a changing reference. We therefore compare two $K=4$ branches
with the same actor-update count and local coefficient $h=T/4$ at
$T\in\{4,7,10\}$ on the nine tasks. In the fixed-reference branch,
actors $\pi_2,\pi_3,\pi_4$ use $\pi_1$ as their reference; in the
re-centered branch, each uses its predecessor.
Figure~\ref{fig:recenter_k4}
shows a mean re-centered-minus-fixed difference of $-0.73$ points,
with a task-bootstrap interval of $[-10.6,8.1]$.
The task means range from $+24.9$ on Hopper-medium to $-32.8$
on HalfCheetah-expert, showing that the return effect depends
strongly on the environment.
In a stagewise diagnostic on two HalfCheetah-expert seeds, the
matched branches share the critic and $\pi_1$. The training
references differ beginning at $\pi_3$. At $T=7$ and $10$,
the fixed-reference branch recovers after low returns at $\pi_2$,
while re-centering continues to reduce return.

\begin{figure}[H]
\centering
\includegraphics[width=.97\textwidth]{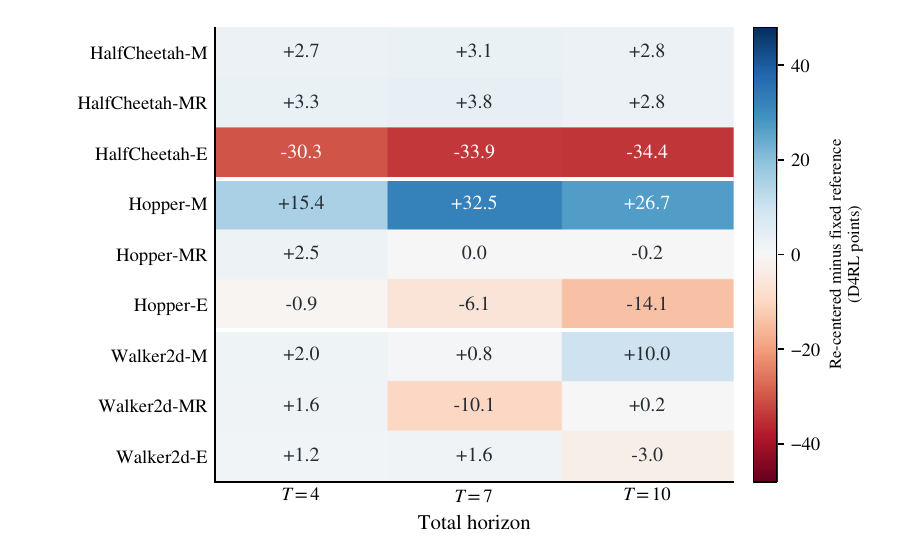}
\caption{Paired I-MPI-4 endpoint differences: re-centered reference
minus fixed reference. M, MR, and E denote medium, medium-replay, and expert.
Positive values favor the re-centered endpoint.}
\label{fig:recenter_k4}
\end{figure}

\subsection{Task-level target and endpoint diagnostics}
\label{app:target_movement}

Figure~\ref{fig:matched_exposure_distribution} shows the
task--horizon--seed ratios underlying the medians in
Figure~\ref{fig:matched_exposure} of Section~\ref{sec:td3bc_diagnostics}.
Across I-MPI-3 and I-MPI-4,
the first target actor has smaller action displacement and
target-value perturbation than TD3+BC on most matched configurations.
For I-MPI-4, the median reduction in target-value perturbation under
the shared TD3+BC critic is $38\%$.
The final actor lies farther from observed successor actions than
its first target actor: the median final-to-target RMS ratios are
$1.297$ for I-MPI-3 and $1.365$ for I-MPI-4.

\begin{figure}[ht]
\centering
\includegraphics[width=\linewidth]{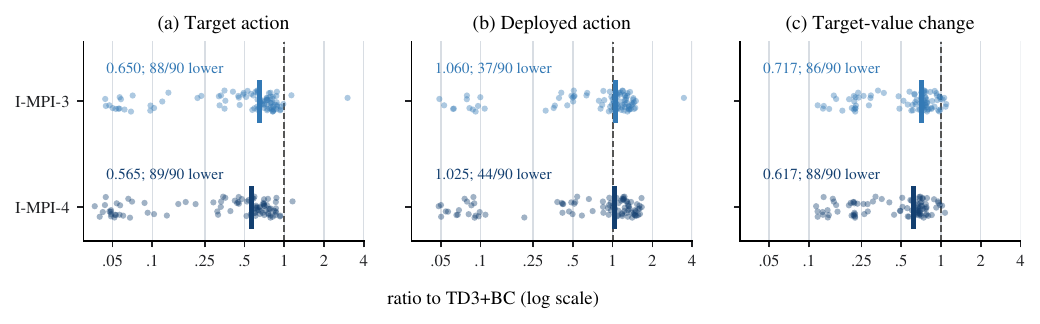}
\caption{Paired task--horizon--seed ratios to TD3+BC at the same total horizon.
Action panels use squared RMS displacement from the observed next dataset
action; the value panel uses the RMS Bellman-target perturbation under the
common TD3+BC critic and identical smoothing noise. Dots are individual
task--horizon--seed cells, vertical bars mark medians, and labels give the
median and number of ratios below one. The logarithmic axis retains the
spread across runs. Ratios include the conditional-action-variance floor.}
\label{fig:matched_exposure_distribution}
\end{figure}

\clearpage
\section{Gaussian \texorpdfstring{\(Q\)+BC--\(W_2\)}{Q+BC--W2} results by task}
\label{app:iql_results}

Figure~\ref{fig:iql_w2_all} shows all nine tasks at the ten horizons
observed for every depth $K=1,\ldots,4$. The selection protocol is
given in Appendix~\ref{app:iql_protocol}.

\begin{figure}[H]
\centering
\includegraphics[width=\textwidth]{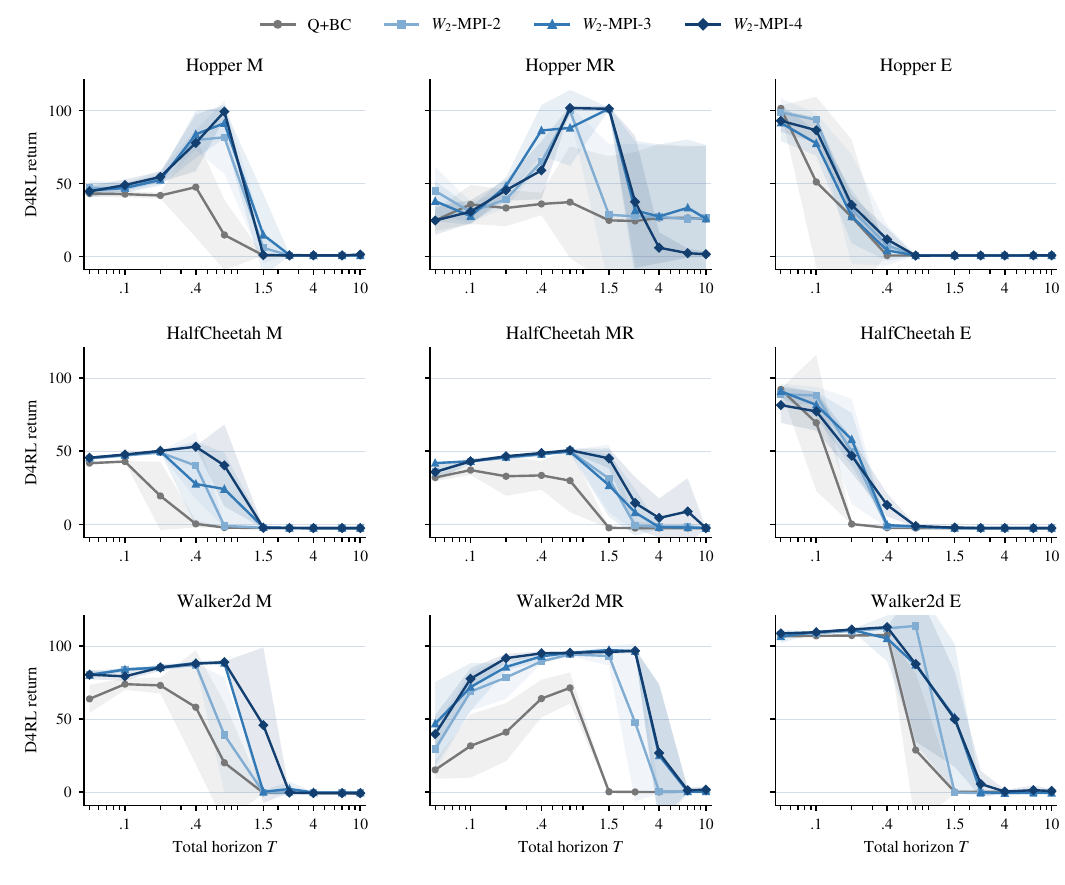}
\caption{Gaussian \(Q\)+BC--Wasserstein, \(K=1\)--\(4\).
Lines connect observed points; bands show one sample standard deviation
across seeds. M/MR/E denote medium/medium-replay/expert.}
\label{fig:iql_w2_all}
\end{figure}

\end{document}